\documentclass[11pt]{article}

\usepackage[margin=1in]{geometry}
\usepackage[T1]{fontenc}
\usepackage[utf8]{inputenc}
\usepackage{lmodern}
\usepackage{microtype}

\usepackage{tabularx}
\usepackage{amsmath,amssymb,amsthm,mathtools}
\usepackage{bm}
\usepackage{booktabs}
\usepackage{enumitem}
\usepackage{array}
\usepackage[hypertexnames=false]{hyperref}
\usepackage[nameinlink,noabbrev]{cleveref}

\usepackage{multirow}
\usepackage{makecell}
\usepackage{threeparttable}
\usepackage[table]{xcolor}
\usepackage{float}
\usepackage{subfig}
\usepackage{authblk}
\usepackage{algorithm}
\usepackage{algorithmic}

\allowdisplaybreaks
\newtheorem{proposition}{Proposition}

\theoremstyle{remark}
\newtheorem{remark}{Remark}

\newcommand{\meanstd}[2]{$#1 \pm \!\mbox{\scriptsize $#2$}$}

\newenvironment{shrinkfix}
{ \bgroup
	\addtolength\abovedisplayshortskip{0ex}
	\addtolength\abovedisplayskip{0ex}
	\addtolength\belowdisplayshortskip{0ex}
	\addtolength\belowdisplayskip{0ex}}
{\egroup\ignorespacesafterend}

\title{Tail-Aware Geometry Learning for Conformal Ellipsoids}
\author{Xiang Zhang,
}

\affil[1]{School of Physical and Mathematical Sciences, Nanyang Technological University, Singapore}

\date{}

\begin{document}

\maketitle

\begin{abstract}
This paper studies multivariate conformal prediction (CP), a distribution-free uncertainty quantification framework with finite-sample coverage guarantees. The efficiency of multivariate prediction sets hinges critically on the residual geometry encoded by the nonconformity score, while existing minimum-volume methods rely on quantile thresholds that ignore tail residual severity and implicitly bind geometry learning to coverage level. We propose a tail-aware geometry learning framework for conformal ellipsoids that decouples tail sensitivity in geometry learning from the final coverage guarantee. Using a two-split design, we learn the metric matrix via volume minimization under a CVaR constraint on an estimation split,  then apply standard conformal calibration on a held-out calibration split. The resulting problem is convex and admits a bounded-reweighting interpretation that prioritizes high-residual samples. Moreover, we theoretically characterize the trade-off between ellipsoidal volume and tail severity. Experimental results demonstrate the effectiveness of the proposed method.

\end{abstract}

\section{Introduction}

Conformal prediction (CP) is a model-agnostic, distribution-free uncertainty quantification framework \cite{vovk2005algorithmic,angelopoulos2023conformal} that yields prediction sets with finite-sample coverage under exchangeability. It has been applied in numerous signal processing tasks, including time series prediction \cite{angelopoulos2023conformal} and anomaly detection \cite{angelopoulos2023conformal}. A core design choice is the nonconformity score, which quantifies how atypical a prediction is relative to observed data. In one dimension, scalar residuals admit a natural total order, so quantiles are well defined. In multivariate regression, however, score design is nontrivial because the score must encode the joint geometry of vector-valued residuals. Consequently, conformal procedures with the same coverage guarantee can produce prediction regions with substantially different shapes and volumes.

Ellipsoidal scores provide a convenient way to model such residue geometry. Common choices include Mahalanobis scores derived from estimated residual covariance \cite{xu2024conformal,braun2025multivariate,messoudi2022ellipsoidal} and its structured variants \cite{dua2026conformal,zhang2026topology},  while recent methods \cite{alon2024learning, braun2025minimum} learn the geometry via minimum-volume objectives 
\cite{henderson2024adaptive,tumu2024multi}. Such formulations are typically governed by an order statistic of the residual scores and can be interpreted as value-at-risk (VaR)-type criteria \cite{duffie1997overview}. Beyond inducing nonconvex optimization, VaR criteria merely delineate a tail boundary: residual severity beyond the active quantile has no impact on the learned geometry. This is limiting when upper-tail residuals exhibit systematic directional structure.

Our key observation is that geometry learning and conformal coverage need not be controlled by the same criterion. The former determines which residual directions shape the score, whereas the latter is enforced afterwards by conformal calibration. We thus use conditional value-at-risk (CVaR) for ellipsoidal geometry learning, which aggregates upper-tail score magnitudes rather than relying solely on the quantile boundary. A tail parameter \(\tau\) controls which residuals are emphasized during learning, while an independent miscoverage level \(\alpha\) determines final coverage, allowing separate tuning of tail sensitivity and coverage. The resulting problem is convex and admits a bounded-reweighting interpretation, where large-score residuals exert greater influence on the learned metric.

CVaR-based minimum-volume ellipsoids have classical precedents \cite{gotoh2008conditional}. Our contribution is instead to use CVaR as a conformal score-learning mechanism followed by an independent calibration step. This also differs from conformal risk-control methods that apply CVaR to downstream losses \cite{yeh2026conformal,chen2026adversarially}: here CVaR shapes prediction-set geometry, with coverage provided by conformal calibration. Finally, our approach differs from existing multivariate CP methods \cite{dheur2025unified}\textemdash including hyperrectangle-based \cite{neeven2018conformal}, density-based \cite{plassier2025probabilistic, sampson2025flexible}, distance-based \cite{xu2024conformal}, copula-based \cite{Sun2022CopulaCP, messoudi2021copula}, and optimal-transport-based approaches \cite{thurin2025optimal}\textemdash in that it employs ellipsoidal scores to encode the joint geometry of vector residuals.

Our contributions are threefold. (\romannumeral1) We formulate tail-aware conformal geometry learning that decouples the tail-sensitivity parameter from the target miscoverage level.  (\romannumeral2) We characterize how upper-tail residuals shape the learned geometry, and derive a convex formulation of the proposed model along with its bounded-reweighting interpretation. (\romannumeral 3) We conduct extensive experiments to demonstrate the effectiveness of the proposed method.

\section{Conformal Ellipsoids: Geometry, Calibration and Coverage}
\label{sec:preliminaries}
Let $\widehat f$ be a predictor trained independently of the calibration
data, and let $\mathcal D_{\mathrm{est}}=\{(X_m,Y_m)\}_{m=1}^M$ and
$\mathcal D_{\mathrm{cal}}=\{(X_n,Y_n)\}_{n=1}^N$
denote disjoint geometry-estimation and calibration sets.
We consider ellipsoidal nonconformity function of the form $S_{\Omega}(X,Y)= r^\top\Omega r$, where $r=Y-\widehat f(X)$ is the residue vector of a sample $(X,Y)$. The metric matrix $\Omega$ determines the relative scaling and orientation of the residual space and is estimated using only 
$\mathcal D_{\mathrm{est}}$. For any  matrix $\widehat\Omega$ estimated from $\mathcal D_{\mathrm{est}}$ and fixed before calibration, define the calibration scores \( \widehat s_n (\widehat \Omega)
= S_{\widehat\Omega}(X_n,Y_n)\).  Let \(k_\alpha =\left\lceil (N+1)(1-\alpha)\right\rceil \) and let $\widehat q_\alpha$ be the $k_\alpha$-th order statistic of \( \{\widehat s_1(\widehat \Omega),\ldots,\widehat s_N(\widehat \Omega),+\infty\}\). 
The corresponding conformal prediction ellipsoid set is
\begin{shrinkfix}
	\begin{align}
		\mathcal Q_{\alpha,\widehat\Omega}(X)
		=
		\left\{
		Y:
		S_{\widehat\Omega}(X,Y)
		\le
		\widehat q_\alpha
		\right\}.
		\label{eq:conformal-set}
	\end{align}
\end{shrinkfix}
The prediction set \eqref{eq:conformal-set} can provide the following coverage validity for test sample $(X_{N+1}, Y_{N+1})$.
\begin{proposition}
	\label{prop:generic_coverage}
	Let $\mathcal F$ contains all information fixed before accessing the calibration and test observations, including
	$\widehat f$ and the information used to estimate
	$\widehat\Omega$.
	If the calibration and test observations are
	exchangeable conditional on $\mathcal F$, 
	\begin{shrinkfix}
		\begin{align}
			\mathbb P
			\left(
			Y_{N+1}
			\in
			\mathcal Q_{\alpha,\widehat\Omega}(X_{N+1})
			\,\middle|\,
			\mathcal F 
			\right)
			\ge
			1-\alpha .
			\label{eq:generic-coverage}
		\end{align}
	\end{shrinkfix}
\end{proposition}
This proposition follows directly from the standard split-conformal argument, and states that a calibration procedure delivers coverage validity once the estimated \(\widehat\Omega\) is fixed. Importantly, conformal validity does not require the geometry-learning objective itself to enforce empirical \(1-\alpha\) coverage.

A second useful property is scale invariance. For every $c>0$, \(S_{c\Omega}(X,Y)= c\,S_\Omega(X,Y)\) and \( \widehat q_\alpha(c\Omega)=c\,\widehat q_\alpha(\Omega)\). Hence, \(\mathcal Q_{\alpha,c\Omega}(X)=\mathcal Q_{\alpha,\Omega}(X).\)  Thus, conformal calibration makes the final prediction set invariant to positive global rescaling of \(\Omega\). Only the relative directional scaling and orientation encoded by \(\Omega\) affect its geometry. In the geometry-learning stage below, the unit threshold therefore serves only as a normalization.

\section{Tail-Aware Geometry Learning}
\label{sec:method}

This section first introduces the formulation of order-statistic geometry learning for CP. We then present our proposed method and conclude with the conformal recalibration.

\subsection{Order-Statistic Geometry Learning}
\label{subsec:var_geometry}
Given \( \mathcal D_{\mathrm{est}}\), let \(r_m
=Y_m-\widehat f(X_m)\) and \(s_m(\Omega)=r_m^\top\Omega r_m\).
The normalized unit-threshold ellipsoid has volume
\begin{shrinkfix}
	\begin{align}
		\operatorname{Vol}(\mathcal E_\Omega)
		=
		\lambda(\mathbb B_d)
		\det(\Omega)^{-1/2},
		\quad
		\mathcal E_\Omega
		=
		\{r:r^\top\Omega r\le1\},
		\label{eq:normalized-volume}
	\end{align}
\end{shrinkfix}
where $\mathbb B_d$ denotes the unit Euclidean ball in
$\mathbb R^d$ and $\lambda$ denotes the Lebesgue measure. To distinguish geometry learning from conformal calibration, we
introduce a tail fraction \( \tau\in(0,1)\),  which is independent of the final miscoverage level $\alpha$.  After removing constant terms, traditional minimum-volume geometry learning \cite{braun2025minimum} learns $\Omega$ by minimizing the log-volume of this normalized ellipsoid:
\begin{shrinkfix}
	\begin{align}
		\min_{\Omega}\,
		&
		-\frac12\log\det\Omega \label{P-O}
		\\
		\mathrm{s.t.}\,
		& \operatorname{Card} \{m\in [M]: s_m(\Omega) \leq 1 \} \geq \left\lceil(1-\tau)M\right\rceil,\;
		\Omega\succ 0,
		%\delta I \preceq \Omega \preceq \Delta I ,
		\notag 
	\end{align}
\end{shrinkfix}
where $\operatorname{Card}(\cdot)$ denotes the cardinality functions, and $[M]:=\{1,\cdots,M\}$. 
%The constants $0<\delta< \Delta$ impose lower and upper spectral bounds on $\Omega$, ensuring positive definiteness and preventing unbounded solutions. 
The first constraint is an order-statistic to ensure an empirical $1-\tau$ inclusion requirement on $\mathcal{D}_{\mathrm{est}}$. Thus, problem \eqref{P-O} is an empirical quantile-constrained minimum-volume geometry learner. Conventional methods \cite{alon2024learning, braun2025minimum} typically set \(\tau = \alpha\), which implicitly couples the geometry learning to the final coverage level via the order-statistic constraint. However, Proposition~\ref{prop:generic_coverage} shows that this is not required for conformal validity: the geometry parameter and the final coverage level play different roles.

We next show that problem \eqref{P-O} can be reformulated as a VaR constrained problem. Define the empirical VaR as 
\begin{shrinkfix}
	\begin{align}
		V_\tau(\Omega)
		=
		\widehat{\operatorname{VaR}}_{1-\tau}
		\left(
		s_1(\Omega),\ldots,s_M(\Omega) 
		\right) = s_{(k_{\tau})}(\Omega)
		\label{eq:var-tau}
	\end{align}
\end{shrinkfix}
where $k_{\tau}=\left\lceil(1-\tau)M\right\rceil$ and $s_{(k_{\tau})}(\Omega)$ is the $k_{\tau}-$th smallest value of all scores of $\mathcal{D}_{\mathrm{est}}$. Then, the empirical inclusion constraint in \eqref{P-O} is equivalent to $V_\tau(\Omega) \leq 1$, leading to the following reformulation
\begin{shrinkfix}
	\begin{align}
		\min_{\Omega}\quad
		-\frac12\log\det\Omega, \quad
		\mathrm{s.t.}\;
		V_\tau(\Omega)\le1,
		\;
		\Omega\succ 0.
		%\delta I\preceq\Omega\preceq\Delta I .
		\label{eq:P-V-tau}
	\end{align}
\end{shrinkfix}
Nevertheless, VaR retains two limitations for geometry learning. First, $V_\tau(\Omega)$ is an intermediate order statistic and generally induces a nonconvex optimization problem.  Second, once an observation lies beyond the active quantile, its additional score magnitude does not affect the VaR constraint.  Consequently, geometrically distinct upper-tail
residuals may have similar influence even when their severities differ
substantially.

\subsection{CVaR Geometry Learning}
\label{subsec:cvar_geometry}
In this study, we learn the ellipsoidal geometry using the empirical
CVaR at tail fraction $\tau$, which is defined as
\begin{shrinkfix}
	\begin{align}
		C_\tau(\Omega) :&= \widehat{\text{CVaR}}_{1-\tau}
		\bigl(
		s_1(\Omega),\ldots,s_M(\Omega)
		\bigr) 		\label{eq:cvar-definition}\\
		&={\frac{1}{\tau M}\left( \sum_{i=k_{\tau}+1}^M s_{(i)}(\Omega) + \big(k_{\tau} - (1-\tau)M\big) s_{(k_{\tau})}(\Omega) \right)}.
		\notag
	\end{align}
\end{shrinkfix}
It is observed from \eqref{eq:cvar-definition} that CVaR quantifies the average magnitude in the upper tail of the score distribution, instead of only locating the tail boundary like VaR. This avoids relying only on the single order statistic like \eqref{eq:P-V-tau}. We then estimate the geometry through
\begin{shrinkfix}
	\begin{align}
		\min_{\Omega}\quad
		-\frac12\log\det\Omega,\quad
		\mathrm{s.t.}\;
		C_\tau(\Omega)\le1,
		\;
		\Omega\succ 0.
		%\delta I\preceq\Omega\preceq\Delta I .
		\label{eq:P-C-tau}
	\end{align}
\end{shrinkfix}
Using the Rockafellar--Uryasev representation \cite{rockafellar2000optimization}, the empirical CVaR \eqref{eq:cvar-definition} is equivalent to 
\begin{shrinkfix}
	\begin{align}
		C_\tau(\Omega)
		=
		&\min_{t,\{\xi_m\}_{m=1}^M}\quad
		t+
		\frac{1}{\tau M}
		\sum_{m=1}^M\xi_m
		\notag\\
		&\mathrm{s.t.}\quad
		\xi_m
		\ge
		r_m^\top\Omega r_m-t,
		\quad
		\xi_m\ge0,
		\quad
		m\in[M].
		\label{eq:cvar-epi}
	\end{align}
\end{shrinkfix}
Equivalently, the problem \eqref{eq:P-C-tau} can be rewritten as 
\begin{shrinkfix}
	\begin{align}
		\min_{\Omega,t,\xi}\quad&
		-\frac12\log\det\Omega
		\notag\\
		\mathrm{s.t.}\quad&
		t+
		\frac{1}{\tau M}
		\sum_{m=1}^M\xi_m
		\le1,\;\;	\Omega\succ 0
		%\delta I\preceq\Omega\preceq\Delta I,
		\notag\\
		&
		\xi_m\ge
		r_m^\top\Omega r_m-t,\;\;
		\xi_m\ge0, \;\; m \in [M]
		\label{eq:P-C-epi-tau}
	\end{align}
\end{shrinkfix}
Since $r_m^\top\Omega r_m$ is affine in $\Omega$, the objective is
convex and all constraints define convex sets.
Problem~\eqref{eq:P-C-epi-tau} can therefore be solved to global
optimality using standard convex optimization methods
\cite{boyd2004convex}.

Here, the CVaR ellipsoid is not used directly as the final uncertainty region. Instead, $\widehat\Omega_\tau$ defines the geometry of the nonconformity score, while the final conformal threshold is determined independently on the calibration split at level $\alpha$. Thus, $\tau$ controls which upper-tail residuals shape the learned geometry, whereas $\alpha$ controls the final conformal miscoverage level.

\begin{remark}
	The learned $\widehat{\Omega}$ here is a global metric shared across all residuals. It can be extended to locally adaptive variants, along the lines of \cite{alon2024learning, braun2025minimum}, which are based on the VaR criterion.  However, this falls outside the scope of this work and is left for future research.
\end{remark}

\noindent\textbf{Bounded-reweighting interpretation.}
The empirical CVaR admits the dual representation
\cite{rockafellar2002conditional}
\begin{shrinkfix}
	\begin{align}
		C_\tau(\Omega)
		= \max_{\eta\in\mathcal P_\tau}
		\langle\Omega,\Sigma_\eta\rangle.
		\label{eq:cvar-dual}
	\end{align}
\end{shrinkfix}
where
\begin{shrinkfix}
	\begin{align}
		\mathcal P_\tau
		=
		\left\{
		\eta\ge0:
		\mathbf 1^\top\eta=1,\;
		\eta_m\le\frac{1}{\tau M}
		\right\}, \;\Sigma_\eta
		=
		\sum_{m=1}^M
		\eta_m r_mr_m^\top.
		\label{eq:cvar-weight-set}
	\end{align}
\end{shrinkfix}
The corresponding problem \eqref{eq:P-C-tau} can be rewritten as 
\begin{shrinkfix}
	\begin{align}
		\widehat{\Omega}_{\tau} \in\; &\arg\min_{\Omega\succ 0}\;
		-\frac12\log\det\Omega,\notag\\
		&\mathrm{s.t.} \quad
		C_{\tau}(\Omega) = \max_{\eta\in\mathcal P_\tau}
		\langle\Omega,\Sigma_\eta\rangle \leq 1.
		\label{eq-dual-cvar}
	\end{align}
\end{shrinkfix}
Then, for problem \eqref{eq-dual-cvar}, we have the following proposition.
\begin{proposition}
	\label{prop:minvol-cvar}
	Suppose that \eqref{eq-dual-cvar} admits an optimum  and that the empirical residuals span $\mathbb R^d$. Then, there exists an optimal weight $\eta^\star \in \arg\max_{\eta\in\mathcal P_\tau}\left\langle\widehat\Omega_\tau,\Sigma_\eta\right\rangle$ such that 
	\begin{shrinkfix}
		\begin{align}
			\widehat\Omega_\tau^{-1}
			=
			d\Sigma_{\eta^\star}
			=
			d\sum_{m=1}^{M}
			\eta_m^\star r_m r_m^\top.
			\label{eq:omega-inverse-main}
		\end{align}
	\end{shrinkfix}
\end{proposition}
\begin{proof}
	See Appendix in the full paper.
\end{proof}

Proposition~\ref{prop:minvol-cvar} provides an explicit link between tail-aware geometry learning and minimum-volume ellipsoidal geometry. Specifically, CVaR identifies a tail-focused residual scatter matrix through the optimal weights $\eta^\star$, where the maximizing weights place greater emphasis on observations with large current scores. The parameter $\tau$ controls the concentration of this reweighting:  a smaller $\tau$ increases the allowable per-sample weight $1/(\tau M)$ and permits greater concentration on a smaller subset of upper-tail residuals. The minimum-volume objective then converts the resulting tail-weighted scatter matrix into the optimal precision metric via $\widehat{\Omega}_\tau^{-1}=d\Sigma_{\eta^\star}$. Thus, the orientation and relative axis lengths of the ellipsoid are determined by the active tail-weighted residual geometry rather than the unweighted covariance.

\subsection{Conformal Recalibration}

After estimating $\widehat\Omega_\tau$ from
$\mathcal D_{\mathrm{est}}$, we keep it fixed and compute \( \widehat s_n = r_n^\top \widehat\Omega_\tau r_n\)
on the independent calibration set $\mathcal D_{\mathrm{cal}}=\{(X_n,Y_n)\}_{n=1}^N$. The finite-sample corrected
conformal quantile $\widehat q_\alpha$ is then computed at
miscoverage level $\alpha$, and the final prediction region is
\begin{shrinkfix}
	\begin{align}
		\mathcal Q_{\alpha,\tau}(X)
		=
		\left\{
		Y:
		(Y-\widehat f(X))^\top
		\widehat\Omega_\tau
		(Y-\widehat f(X))
		\le
		\widehat q_\alpha
		\right\}.
		\label{eq:final-set}
	\end{align}
\end{shrinkfix}
By Proposition~\ref{prop:generic_coverage}, for a new test sample $(X_{N+1}, Y_{N+1})$,
\begin{shrinkfix}
	\begin{align}
		\mathbb P
		\left(
		Y_{N+1}
		\in
		\mathcal Q_{\alpha,\tau}(X_{N+1})
		\right)
		\ge
		1-\alpha ,
		\label{eq:final-coverage}
	\end{align}
\end{shrinkfix}
provided that $\widehat\Omega_\tau$ is fixed before the calibration data are used.
Therefore, changing $\tau$ alters the learned geometry without
changing the nominal conformal coverage level, while changing
$\alpha$ rescales the final prediction region without requiring the
geometry to be relearned.

\begin{algorithm}[t]
	\caption{CVaR-based Ellipsoidal Conformal Prediction
	}
	\label{alg}
	\begin{algorithmic}[1]
		\REQUIRE   
		Trained predictor $\widehat f$,
		estimation set $\mathcal D_{\mathrm{est}}$,
		calibration set $\mathcal D_{\mathrm{cal}}$,
		tail parameter $\tau$,
		and miscoverage level $\alpha$.
		\STATE Estimate $\widehat{\Omega}_{\tau}$ using
		$\mathcal D_{\mathrm{est}}$ by solving \eqref{eq:P-C-epi-tau}
		\STATE Compute $\widehat s_n(\widehat{\Omega}_{\tau})=r_n^\top\widehat{\Omega}_{\tau}r_n $ for samples in $\mathcal D_{\mathrm{cal}}$
		
		\STATE Compute the conformal  quantile $
		\widehat{q}_{\alpha}$ based on $\{\widehat s_n\}_{n=1}^N$.

		\STATE Construct the prediction set for $X_{N+1}$ based on \eqref{eq:final-set}
		
		\ENSURE The prediction set $Q_{\alpha,\tau}(X_{N+1})$ for $X_{N+1}$
		
	\end{algorithmic}
\end{algorithm}

\section{Geometry--Tail Characterization}
\label{sec:theory}

In this section, we theoretically characterize the trade-off between ellipsoidal volume and tail risk severity for the proposed method. First, for any $\Omega$ with $V_\tau(\Omega)>0$, define the tail-severity ratio
\begin{shrinkfix}
	\begin{align}
		\Gamma_\tau(\Omega)
		=
		\frac{C_\tau(\Omega)}
		{V_\tau(\Omega)}
		\ge1.
		\label{eq:gamma-tau}
	\end{align}
\end{shrinkfix}
This ratio measures the relative severity of scores beyond the VaR boundary: a value close to one indicates that upper-tail scores remain near the boundary, whereas a larger value indicates more severe tail excursions.  Let $\Omega_V^\star$ and $\Omega_C^\star$ denote optimal solutions of
the VaR-based \eqref{eq:P-V-tau} and CVaR-based
\eqref{eq:P-C-tau} geometry-learning problems at the same tail fraction $\tau$, respectively, and define
\begin{shrinkfix}
	\begin{align}
		\Gamma_{\tau,V}
		=
		\Gamma_\tau(\Omega_V^\star),
		\qquad
		\Gamma_{\tau,C}
		=
		\Gamma_\tau(\Omega_C^\star),
	\end{align}
\end{shrinkfix}
as the tail-severity ratio of $\Omega_V^\star$ and $\Omega_C^\star$, respectively. We further introduce the VaR-normalized geometry ellipsoids
\begin{shrinkfix}
	\begin{align}
		\mathcal E_{\tau,j} = \left\{ r: r^\top\Omega_j^\star r \le V_\tau(\Omega_j^\star) \right\},
		\qquad j\in\{V,C\}.
		\label{eq:var-normalized-ellipsoids}
	\end{align}
\end{shrinkfix}
The set $\mathcal E_{\tau,j}$ defines an ellipsoidal region induced by the metric matrix $\Omega_j^\star$,  with its overall scale determined by the threshold $V_\tau(\Omega_j^\star)$. Hence, $\Omega_j^\star$  characterizes the shape and orientation of the ellipsoid, while $V_\tau(\Omega_j^\star)$ controls its size at risk level $\tau$. Then, we have the following proposition.
\begin{proposition}
	\label{thm:geometry-tail-tradeoff}
	Suppose that problems \eqref{eq:P-V-tau} and \eqref{eq:P-C-tau} admit optimal solutions, and that the empirical residuals span $\mathbb R^d$. Then the ellipsoids associated with $\Omega_V^\star$ and $\Omega_C^\star$ satisfy
	\begin{shrinkfix}
		\begin{align}
			1
			\le
			\frac{
				\operatorname{Vol}(\mathcal E_{\tau, C})
			}{
				\operatorname{Vol}(\mathcal E_{\tau, V})
			}
			\le
			\left(
			\frac{
				\Gamma_{\tau,V}
			}{
				\Gamma_{\tau,C}
			}
			\right)^{d/2}.
			\label{eq:geometry-tail-bound}
		\end{align}
	\end{shrinkfix}
	If
	\(
	\operatorname{Vol}(\mathcal E_{\tau,C})
	=
	\rho\,
	\operatorname{Vol}(\mathcal E_{\tau,V})
	\)
	for $\rho\ge1$, then
	\begin{shrinkfix}
		\begin{align}
			\Gamma_{\tau,C}
			\le
			\rho^{-2/d}
			\Gamma_{\tau,V}.
			\label{eq:gamma-volume-relation}
		\end{align}
	\end{shrinkfix}
\end{proposition}

\begin{proof}
	See Appendix in the full paper.
\end{proof}
Proposition~\ref{thm:geometry-tail-tradeoff} quantifies the trade-off between geometric compactness and tail sensitivity. The VaR solution provides the smallest VaR-normalized volume, while the CVaR solution achieves a no-worse tail-severity ratio, $\Gamma_{\tau,C}\le\Gamma_{\tau,V}$. Moreover, any increase in the CVaR-based normalized volume must be compensated by a corresponding reduction in tail severity according to \eqref{eq:gamma-volume-relation}. Thus, CVaR-based learning may allocate more geometric capacity to directions containing severe residual excursions rather than merely optimizing the VaR boundary. This result provides a theoretical rationale for using CVaR to learn a more tail-aware residual geometry. Importantly, it concerns the geometry-estimation stage and does not directly imply a smaller final conformal region, whose scale is determined by the miscoverage level $\alpha$ during calibration. Accordingly, Section~\ref{sec:numerical experiments} evaluates whether this geometry-level tail improvement translates into favorable calibrated volume behavior at coverage level $1-\alpha$.

\section{Experiments}
\label{sec:numerical experiments}

\begin{table}[t]
	\centering \small
	\caption{Experimental results on synthetic data.}
	\label{tab:synthetic data}
	\setlength{\tabcolsep}{15pt}
	\renewcommand{\arraystretch}{1.2}
	\begin{tabular}{ccccc}
		\toprule
		$\alpha$  & Method & Coverage & Efficiency $\downarrow$  & Severity $\downarrow$  \\
		\midrule
		\multirow{5}{*}{\rotatebox{90}{0.1}}
		
		& \textsc{Naive}  &  \meanstd{0.9002}{0.0065}  & \meanstd{3.9105}{0.0247} & \meanstd{1.2434}{0.0219}  \\
		&\textsc{Cov}  &  \meanstd{0.9002}{0.0055} &  \meanstd{3.6929}{0.0159} &  \meanstd{1.1854}{0.0139} \\
		&  \textsc{MVCS}  & \meanstd{0.8998}{0.0060} &  \meanstd{\textbf{3.6081}}{0.0166} &  \meanstd{1.1893}{0.0151}  \\
		\rowcolor{gray!20}
		& \textsc{Ours}: $ 0.1$  &   \meanstd{0.9000}{0.0059} & \meanstd{3.6128}{0.0160}&  \meanstd{1.1691} {0.0144}\\
		\rowcolor{gray!20}
		&\textsc{Ours}: $ 0.08$  &   \meanstd{0.9001}{0.0059} & \meanstd{\underline{3.6082}}{0.0156}&  \meanstd{ \underline{1.1665}}{0.0163}\\
		\rowcolor{gray!20}
		&\textsc{Ours}: $ 0.05$  &   \meanstd{0.9001}{0.0061} & \meanstd{3.6250}{0.0160}&  \meanstd{\textbf{ 1.1652}}{0.0177}\\
		\midrule
		\multirow{5}{*}{\rotatebox{90}{0.05}}
		& \textsc{Naive}  &  \meanstd{0.9493}{0.0050}  & \meanstd{4.1995}{0.0305} & \meanstd{1.1940}{0.0244}  \\
		&\textsc{Cov}  &  \meanstd{0.9501}{0.0040} &  \meanstd{3.8849}{0.0183} &  \meanstd{1.1562}{0.0173} \\
		&  \textsc{MVCS}  & \meanstd{0.9499}{0.0046} &  \meanstd{\underline{3.8430}}{0.0185} &  \meanstd{1.1664}{0.0193}  \\
		\rowcolor{gray!20}
		&\textsc{Ours}: $0.1$ &   \meanstd{0.9501}{0.0043} & \meanstd{3.8455}{0.0180}&  \meanstd{1.1551}{0.0179} \\
		\rowcolor{gray!20}
		& \textsc{Ours}: $0.08$ & \meanstd{0.9501}{0.0043} & \meanstd{\underline{3.8430}}{0.0183} & \meanstd{\underline{1.1541}}{0.0185} \\
		\rowcolor{gray!20}
		& \textsc{Ours}: $0.05$  &   \meanstd{0.9501}{0.0045} & \meanstd{\textbf{ 3.8403}}{0.0192}&  \meanstd{\textbf{1.1534}}{0.0190} \\
		\bottomrule
		\multicolumn{5}{l}{\scriptsize $\downarrow$: lower is better; bold: best result; underlined: second-best result.}
	\end{tabular}
\end{table}

In this section, we evaluate the proposed method on synthetic and real-world data.

\subsection{Synthetic Data.}
We consider a bivariate regression problem with
$X\sim\mathrm{Uniform}(0,1)$ and
$Y=f^*(X)+A(X)\epsilon\in\mathbb{R}^2$, where
$f_1^*(X)=X+\sin(2\pi X)$ and
$f_2^*(X)=X^2+\cos(2\pi X)$.
The noise is
$\epsilon=(z_1,z_2+0.4z_1^2-0.4\sigma_1^2)$,
where $z_1\sim\mathcal{N}(0,0.1)$,
$z_2\sim\mathcal{N}(0,0.02)$, and $\sigma_1^2$ is the variance of $z_1$.
To introduce a rare anisotropic high-variance regime, we set
$A(X)=I$ for $X\le0.92$ and
\(
A(X)=R_{\pi/4}
\begin{pmatrix}
	1&0\\
	0&5
\end{pmatrix}
R_{\pi/4}^{\top}
\)
otherwise, where $R_{\pi/4}$ is the rotation matrix with angle $\pi/4$. Thus, the amplified regime occurs with probability $0.08$. We generate 20,000 samples and split them into training, estimation,
calibration, and test sets with proportions $0.4$, $0.2$, $0.2$, and $0.2$. A two-layer MLP is fitted on the training set, while the estimation and calibration sets are used for score construction and split conformal
calibration, respectively. We compare \textsc{Ours} with \textsc{Naive}, which uses the $\ell_2$ norm;
\textsc{Cov}, which uses the Mahalanobis distance based on empirical covariance \cite{xu2024conformal}; and \textsc{MVCS}, which learns a minimum-volume ellipsoid under a VaR constraint \cite{braun2025minimum}. For our method, we let $\tau = 0.1,0.08$ and $0.05$, respectively. The evaluation metrics are Coverage and Efficiency, where Efficiency is defined as the $d$-th root of the conformal ellipsoid volume. Moreover, we employ a new metric termed Severity:
\begin{shrinkfix}
	\begin{align}
		\mathrm{Severity}= \widehat{\operatorname{CVaR}}_{1-\alpha}(\overline{s}_1, \cdots, \overline{s}_J)
	\end{align}
\end{shrinkfix}
where $J$ is the number of test samples. Severity measures the average magnitude of the worst \(\alpha\)-tail of normalized scores \(\overline{s}_j = s_j / \widehat{q}_\alpha\) and therefore quantifies how severe extreme prediction outcomes are relative to the calibrated boundary. It can also be interpreted as an empirically calibrated counterpart to the tail-severity ratio \(\Gamma_{\alpha}\) in Proposition \ref{thm:geometry-tail-tradeoff} across calibration and test stages.

Table~\ref{tab:synthetic data} reports results averaged over 100 simulations.  All methods achieve empirical coverage close to the nominal \(1-\alpha\), consistent with Proposition~1 that the independent calibration step ensures coverage guarantees for any fixed \(\widehat{\Omega}\). At $\alpha = 0.1$, \textsc{MVCS} and our method have nearly identical efficiency, while ours achieves lower tail severity. Notably, $\tau = 0.08$ yields the best efficiency, matching the  8\% high-variance population. Since this proportion is below the allowed miscoverage level, a VaR-based objective can favor compact sets while underweighting severe tail losses, whereas CVaR remains sensitive beyond the target quantile. At $\alpha = 0.05$, the allowed miscoverage falls below the rare-regime probability, requiring the geometry to accommodate part of the high-variance population. Our method then achieves both the smallest sets and the lowest severity, with $\tau = 0.05$ performing best on average. This highlights the benefit of tail-aware geometry learning when the coverage constraint extends into the rare regime. The gap may also partly reflect the non-convex optimization of the VaR-based formulation. Overall, CVaR maintains competitive efficiency while providing better tail control.

Fig.\ref{fig:visualization} further provides the visualization of the obtained residues and the coverage area of MVCS and our method. We let $\alpha = 0.1$ and $\tau = 0.1$. We can observe that, compared with MVCS, our method produces a conformal region that is better aligned with the residual distribution. By explicitly emphasizing high-loss samples in the tail, the CVaR objective improves sensitivity to extreme residuals and encourages a more risk-aware region.

\begin{figure}[t]
	\centering
	\subfloat[]{
		\includegraphics[width=0.45\linewidth]{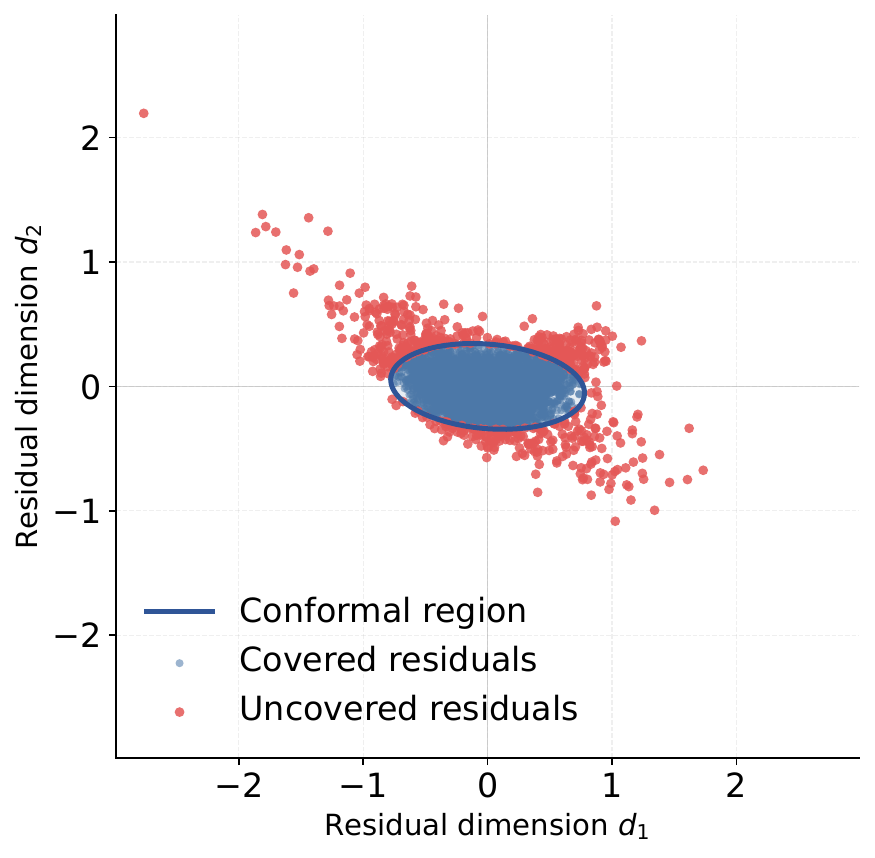}
	}
	\hfill
	\subfloat[]{
		\includegraphics[width=0.45\linewidth]{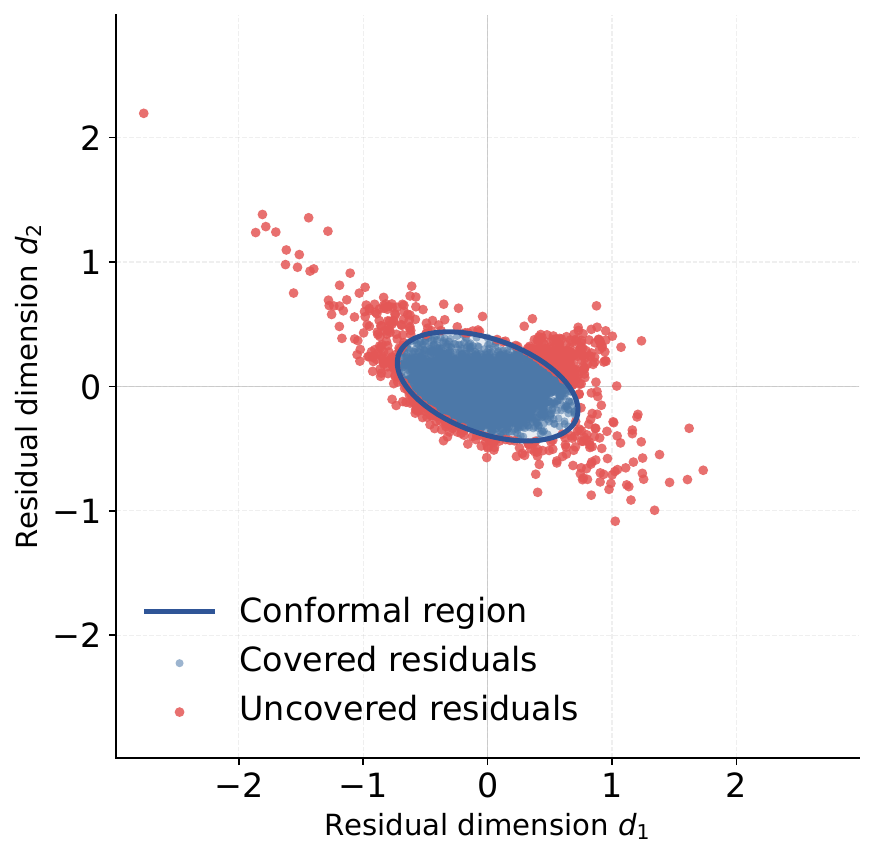}
	}
	\caption{Visualization of the residuals and coverage: (a) MVCS; (b) Our method.}
	\label{fig:visualization}
\end{figure}

\subsection{Real Data.}
We further test the proposed method on three signal processing benchmark datasets: Gas \cite{gas2019uci}, Protein \cite{protein2013uci}, and House \cite{feldman2023calibrated}. These datasets are used for multivariate regression, with a two-layer MLP serving as the base predictor. Further details regarding the datasets and the base predictor are provided in the Appendix of the full paper due to space constraints. We let $\alpha = 0.1$, and the results are reported in Table~\ref{tab:real data}. Across all datasets and choices of $\tau$, the empirical coverage remains close to the nominal level of $0.9$, consistent with the marginal coverage guarantee of Proposition \ref{prop:generic_coverage}. Compared with \textsc{MVCS}, our method incurs only a modest increase in prediction-set size while substantially reducing miscoverage severity.  Moreover, decreasing $\tau$ from $0.1$ to $0.05$ further reduces severity on all three datasets, at the cost of a  decrease in efficiency, while leaving coverage essentially unchanged. This behavior illustrates the role of $\tau$ in controlling the efficiency--severity trade-off independently of the target coverage level $\alpha$. Overall, the results show that the proposed CVaR-based construction preserves conformal coverage and competitive set efficiency while providing substantially improved control over severe miscoverage.

\begin{table}[t]
	\centering \small
	\caption{Results on real data: $\alpha = 0.1$.}
	\label{tab:real data}
	\setlength{\tabcolsep}{15pt}
	\renewcommand{\arraystretch}{1.2}
	\begin{tabular}{ccccc}
		\toprule
		& Method & Coverage & Efficiency $\downarrow$  & Severity $\downarrow$  \\
		\midrule
		\multirow{4}{*}{\rotatebox{90}{Gas}}
		& \textsc{Naive}  &  \meanstd{0.8980}{0.0047}  & \meanstd{12.4092}{0.1967} & \meanstd{3.9242}{0.1770}  \\
		&\textsc{Cov}  &  \meanstd{0.8982}{0.0058} &  \meanstd{8.5057}{0.1217} &  \meanstd{4.0354}{0.4502} \\
		&  \textsc{MVCS}  & \meanstd{0.8992}{0.0058} &  \meanstd{\textbf{ 7.5024}}{0.1229} &  \meanstd{4.0410}{0.4537}  \\
		\rowcolor{gray!20}
		& \textsc{Ours}: $0.1$  &   \meanstd{0.8990}{0.0053} & \meanstd{\underline{7.7740}}{0.1287}&  \meanstd{\underline{ 2.2924}}{0.1595} \\
		\rowcolor{gray!20}
		& \textsc{Ours}: $0.05$  &   \meanstd{0.8994}{0.0056} & \meanstd{7.8556}{0.1298}&  \meanstd{\textbf{1.9917}}{0.1848} \\

		\midrule
		\multirow{4}{*}{\rotatebox{90}{Protein}}
		& \textsc{Naive}  &  \meanstd{0.9000}{0.0063}  & \meanstd{2.9395}{0.2406} & \meanstd{7.7390}{2.6724}  \\
		&\textsc{Cov}  &  \meanstd{0.9006}{0.0064} &  \meanstd{2.6816}{0.1973} &  \meanstd{5.4533}{2.4981} \\
		&  \textsc{MVCS}  & \meanstd{0.9007}{0.0061} &  \meanstd{\textbf{2.5095}}{0.1924} &  \meanstd{5.5840}{2.5648}  \\
		\rowcolor{gray!20}
		& \textsc{Ours}: $0.1$  &   \meanstd{0.9007}{0.0056} & \meanstd{\underline{2.5316}}{0.1976}&  \meanstd{\underline{3.7681}}{1.3876} \\
		\rowcolor{gray!20}
		& \textsc{Ours}: $0.05$  &   \meanstd{0.9002}{0.0063} & \meanstd{2.5771}{0.1992}&  \meanstd{\textbf{3.3327}}{1.5972} \\

		\midrule
		\multirow{4}{*}{\rotatebox{90}{House}}
		& \textsc{Naive}  &  \meanstd{0.9019}{0.0060}  & \meanstd{2.3647}{0.0280} & \meanstd{2.4940}{0.1849}  \\
		&\textsc{Cov}  &  \meanstd{0.9018}{0.0073} &  \meanstd{2.1926}{0.0356} &  \meanstd{2.5127}{0.3122} \\
		&  \textsc{MVCS}  & \meanstd{0.9018}{0.0072} &  \meanstd{\textbf{2.1212}}{0.0354} &  \meanstd{2.5179}{0.3141}  \\
		\rowcolor{gray!20}
		& \textsc{Ours}: $0.1$  &   \meanstd{0.9012}{0.0062} & \meanstd{\underline{2.1458}}{0.0364}&  \meanstd{\underline{2.3191}}{0.2716} \\
		\rowcolor{gray!20}
		& \textsc{Ours}: $0.05$  &   \meanstd{0.9009}{0.0068} & \meanstd{2.1677}{0.03743}&  \meanstd{\textbf{2.2214}}{0.2889} \\
		\bottomrule
	\end{tabular}
\end{table}

\section{Conclusion}

We propose a tail-aware geometry learning framework for multivariate conformal prediction with ellipsoidal scores. Our approach decouples tail sensitivity from the target miscoverage level, using an independent parameter to control the shape of the learned geometry via a CVaR criterion. The resulting formulation is convex and admits a bounded-reweighting interpretation, revealing how upper-tail residuals shape the learned metric. We further characterize the trade-off between tail severity and ellipsoidal volume. Extensive experiments validate the effectiveness of the proposed method.

\bibliographystyle{IEEEtran}
\bibliography{refs}

\vfill\pagebreak
\section{Appendix}
\label{sec:appendix}
\subsection{Proof of Proposition~\ref{prop:minvol-cvar}}
\label{app:minvol-proof}

Recall that
\begin{equation}
	C_\tau(\Omega)
	=
	\max_{\eta\in\mathcal P_\tau}
	\left\langle
	\Omega,\Sigma_\eta
	\right\rangle,
	\qquad
	\Sigma_\eta
	=
	\sum_{m=1}^{M}
	\eta_m r_m r_m^\top,
	\label{eq:app-cvar-dual}
\end{equation}
where
\begin{equation}
	\mathcal P_\tau
	=
	\left\{
	\eta\in\mathbb R^M:
	\eta\ge0,\;
	\mathbf 1^\top\eta=1,\;
	\eta_m\le\frac{1}{\tau M},
	\;m\in[M]
	\right\}.
	\notag
\end{equation}
Throughout the proof, we assume $\tau\in(0,1]$ and that
$\{r_m\}_{m=1}^M$ spans $\mathbb R^d$. We consider the convex optimization problem
\begin{equation}
	\begin{aligned}
		\min_{\Omega\succ0}\;
		-\frac12\log\det\Omega, \quad
		\mathrm{s.t.}\;
		C_\tau(\Omega)\le1.
	\end{aligned}
	\label{eq:app-primal}
\end{equation}
Define the empirical residual scatter matrix as $	S \triangleq \frac1M\sum_{m=1}^M r_m r_m^\top.$ Since $\{r_m\}_{m=1}^M$ spans $\mathbb R^d$, we have $S\succ0.$ Consider the uniform weight vector $	\bar\eta = \left( \frac1M,\ldots,\frac1M \right).$ Since $\tau\in(0,1]$, we have $	\frac1M \le \frac{1}{\tau M},$ and hence $\bar\eta\in\mathcal P_\tau.$ Therefore, by the dual representation \eqref{eq:app-cvar-dual}, we have $	C_\tau(\Omega) \ge \left\langle \Omega,S \right\rangle.$

Let $\mu \triangleq \lambda_{\min}(S)>0.$ For every $\Omega\succeq0$, we obtain $\left\langle \Omega,S \right\rangle\ge \mu\,\operatorname{tr}(\Omega).$ Consequently, every feasible $\Omega$ satisfies
\begin{equation}
	\operatorname{tr}(\Omega)
	\le
	\frac1\mu.
	\label{eq:app-trace-bound}
\end{equation}
Thus the feasible set is bounded in the positive-semidefinite cone. The problem is also strictly feasible. Indeed, for any $\varepsilon>0$, $	C_\tau(\varepsilon I) = \varepsilon C_\tau(I),$ by positive homogeneity of $C_\tau$. Hence, for sufficiently small $C_\tau(\varepsilon I)<1.$ Therefore Slater's condition holds.

To establish attainment of the optimum, let
$\Omega_0\succ0$ be any strictly feasible point and let
$\{\Omega_n\}_{n\ge1}$ be a minimizing sequence such that, without
loss of generality, we have 
\begin{equation}
	-\frac12\log\det\Omega_n
	\le
	-\frac12\log\det\Omega_0.
\end{equation}
Equivalently, this leads to $\det(\Omega_n) \ge \det(\Omega_0)>0.$ Together with the trace bound \eqref{eq:app-trace-bound}, this implies that the eigenvalues of $\Omega_n$ are bounded above and bounded away from zero. Consequently, the minimizing sequence lies in a compact subset of the positive-definite cone. Since both $C_\tau(\Omega)$ and $-\log\det\Omega$ are continuous on this set, an optimal solution $\widehat\Omega_\tau\succ0$ exists. Moreover, the function $\Omega\mapsto-\log\det\Omega$ is strictly convex on the positive-definite cone, while the feasible
set is convex because $C_\tau$ is convex. Hence the optimizer $\widehat\Omega_\tau$ is unique.

We next explore the activity of the CVaR constraint. For every scalar $c>0$, the dual representation \eqref{eq:app-cvar-dual} gives
\begin{align}
	C_\tau(c\Omega) =
	\max_{\eta\in\mathcal P_\tau}
	\left\langle
	c\Omega,\Sigma_\eta
	\right\rangle =
	c
	\max_{\eta\in\mathcal P_\tau}
	\left\langle
	\Omega,\Sigma_\eta
	\right\rangle=
	cC_\tau(\Omega).
	\label{eq:app-homogeneity}
\end{align}
Thus $C_\tau$ is positively homogeneous. Suppose, for contradiction, that $	C_\tau(\widehat\Omega_\tau)<1.$ Then one can choose $c>1$ sufficiently close to one such that $cC_\tau(\widehat\Omega_\tau)\le1.$
By \eqref{eq:app-homogeneity},
$c\widehat\Omega_\tau$ remains feasible.

On the other hand,
\begin{align}
	-\frac12\log\det(c\widehat\Omega_\tau)
	&=
	-\frac12
	\log
	\left(
	c^d\det\widehat\Omega_\tau
	\right)
	\nonumber\\
	&=
	-\frac12\log\det\widehat\Omega_\tau
	-\frac d2\log c
	\nonumber\\
	&<
	-\frac12\log\det\widehat\Omega_\tau,
\end{align}
which contradicts the optimality of
$\widehat\Omega_\tau$. Therefore, we have 
\begin{equation}
	C_\tau(\widehat\Omega_\tau)=1.
	\label{eq:app-active}
\end{equation}

Then, we need to study the active CVaR weights and the subdifferential. Define the active set of CVaR weights at $\widehat\Omega_\tau$ by
\begin{equation}
	\mathcal A_\tau(\widehat\Omega_\tau)
	\triangleq
	\arg\max_{\eta\in\mathcal P_\tau}
	\left\langle
	\widehat\Omega_\tau,\Sigma_\eta
	\right\rangle.
	\label{eq:app-active-set}
\end{equation}
Since $\mathcal P_\tau$ is compact and the maximized objective is
continuous in $\eta$,
$\mathcal A_\tau(\widehat\Omega_\tau)$ is nonempty.
Moreover, because the objective is linear in $\eta$,
$\mathcal A_\tau(\widehat\Omega_\tau)$ is a convex face of
$\mathcal P_\tau$. Because $C_\tau(\Omega)$ is the pointwise maximum of linear
functions of $\Omega$, Danskin's theorem yields
\begin{equation}
	\partial C_\tau(\widehat\Omega_\tau)
	=
	\operatorname{conv}
	\left\{
	\Sigma_\eta:
	\eta\in
	\mathcal A_\tau(\widehat\Omega_\tau)
	\right\}.
	\label{eq:app-danskin}
\end{equation}
Since
$\mathcal A_\tau(\widehat\Omega_\tau)$ is convex and
$\eta\mapsto\Sigma_\eta$ is linear, the image set
\(
\left\{
\Sigma_\eta:
\eta\in
\mathcal A_\tau(\widehat\Omega_\tau)
\right\}
\)is itself convex. Hence, it leads to 
\begin{equation}
	\partial C_\tau(\widehat\Omega_\tau)
	=
	\left\{
	\Sigma_\eta:
	\eta\in
	\mathcal A_\tau(\widehat\Omega_\tau)
	\right\}.
	\label{eq:app-subgradient-active}
\end{equation}
Therefore every subgradient of $C_\tau$ at
$\widehat\Omega_\tau$ can be represented by the scatter matrix
associated with an active CVaR weight vector.

Finally, we derive the KKT conditions. Since Slater's condition holds, the KKT conditions are necessary and
sufficient for \eqref{eq:app-primal}. Therefore there exist a
multiplier $\lambda^\star\ge0$ and a matrix $	G^\star \in \partial C_\tau(\widehat\Omega_\tau)$ such that
\begin{equation}
	-\frac12\widehat\Omega_\tau^{-1} + \lambda^\star G^\star = 0.
	\label{eq:app-kkt}
\end{equation}
The multiplier must satisfy $\lambda^\star>0,$ because $\lambda^\star=0$ would imply $	\widehat\Omega_\tau^{-1}=0,$ which is impossible for
$\widehat\Omega_\tau\succ0$. By \eqref{eq:app-subgradient-active}, there exists $\eta^\star \in \mathcal A_\tau(\widehat\Omega_\tau)$ such that $	G^\star = \Sigma_{\eta^\star}.$ Thus \eqref{eq:app-kkt} becomes
\begin{equation}
	\widehat\Omega_\tau^{-1}
	=
	2\lambda^\star
	\Sigma_{\eta^\star}.
	\label{eq:app-before-normalization}
\end{equation}
Since $\eta^\star$ is active, we have $	\left\langle \widehat\Omega_\tau, \Sigma_{\eta^\star} \right\rangle = C_\tau(\widehat\Omega_\tau).$ Using \eqref{eq:app-active}, we have $	\left\langle \widehat\Omega_\tau, \Sigma_{\eta^\star} \right\rangle = 1.$ Multiplying \eqref{eq:app-before-normalization} by $\widehat\Omega_\tau$ and taking the trace gives
\begin{align}
	d
	&=
	\operatorname{tr}
	\left(
	\widehat\Omega_\tau
	\widehat\Omega_\tau^{-1}
	\right) =
	2\lambda^\star
	\operatorname{tr}
	\left(
	\widehat\Omega_\tau
	\Sigma_{\eta^\star}
	\right)
	\nonumber\\
	&=
	2\lambda^\star
	\left\langle
	\widehat\Omega_\tau,
	\Sigma_{\eta^\star}
	\right\rangle =
	2\lambda^\star.
	\label{eq:app-lambda}
\end{align}

Substituting \eqref{eq:app-lambda} into
\eqref{eq:app-before-normalization} yields
\begin{equation}
	\widehat\Omega_\tau^{-1}
	=
	d\,\Sigma_{\eta^\star}
	=
	d\sum_{m=1}^{M}
	\eta_m^\star r_m r_m^\top.
	\label{eq:app-inverse-final}
\end{equation}
Since $\widehat\Omega_\tau\succ0$,
\eqref{eq:app-inverse-final} implies $\Sigma_{\eta^\star}\succ0.$ Consequently, we obtain 
\begin{equation}
	\widehat\Omega_\tau
	=
	\frac1d
	\Sigma_{\eta^\star}^{-1}
	=
	\frac1d
	\left(
	\sum_{m=1}^{M}
	\eta_m^\star r_m r_m^\top
	\right)^{-1}.
	\label{eq:app-omega-final}
\end{equation}

The active weight vector $\eta^\star$ need not be unique.
Nevertheless, every KKT-compatible active weight vector satisfying
\eqref{eq:app-inverse-final} induces the same scatter matrix, namely
\begin{equation}
	\Sigma_{\eta^\star}
	=
	\frac1d\widehat\Omega_\tau^{-1}.
	\label{eq:app-scatter-unique}
\end{equation}
This proves the proposition.

\subsection{Proof of Proposition~\ref{thm:geometry-tail-tradeoff}}
For the fixed empirical residual sample, define the quadratic score associated with $\Omega\succ0$ as \( s_m(\Omega) = r_m^\top \Omega r_m, m\in[M]. \) 	For any $c>0$, we have $s_m(c\Omega) = c\,s_m(\Omega).$ 	Since multiplication by a positive scalar preserves the ordering of the empirical scores, both the empirical VaR and CVaR are positively homogeneous:
\begin{equation}
	V_\tau(c\Omega)
	=
	cV_\tau(\Omega),
	\qquad
	C_\tau(c\Omega)
	=
	cC_\tau(\Omega),
	\qquad c>0.
	\label{eq:tradeoff-risk-homogeneity}
\end{equation}

We first verify that the tail-severity ratios appearing in the proposition are well defined. In particular, the assumed existence of an optimal solution to the VaR-based problem \eqref{eq:P-V-tau} implies that
\begin{equation}
	V_\tau(\Omega)>0
	\qquad
	\text{for every } \Omega\succ0.
	\label{eq:tradeoff-var-positive}
\end{equation}
Indeed, suppose that there existed some $\widetilde{\Omega}\succ0$ such that $V_\tau(\widetilde{\Omega})=0.$ Then, by \eqref{eq:tradeoff-risk-homogeneity},
we have $ V_\tau(c\widetilde{\Omega}) = cV_\tau(\widetilde{\Omega}) = 0 \le 1 $ 	for every $c>0$. Hence $c\widetilde{\Omega}$ would remain feasible for the VaR-based problem for arbitrarily large $c$. However,
\begin{align}
	-\frac12\log\det(c\widetilde{\Omega})
	&=
	-\frac12
	\log\left(
	c^d\det(\widetilde{\Omega})
	\right)
	\nonumber\\
	&=
	-\frac12\log\det(\widetilde{\Omega})
	-
	\frac d2\log c
	\longrightarrow
	-\infty,
\end{align}
contradicting the assumed existence of an optimal solution to \eqref{eq:P-V-tau}. Therefore, \eqref{eq:tradeoff-var-positive} holds. In particular, $	V_\tau(\Omega_V^\star)>0, V_\tau(\Omega_C^\star)>0,$ and hence both $\Gamma_{\tau,V}$ and $\Gamma_{\tau,C}$ are well defined.

We next show that the corresponding risk constraints are active at the two optima. Suppose, for contradiction, that $V_\tau(\Omega_V^\star)<1.$ 	Since $V_\tau(\Omega_V^\star)>0$, there exists some $c>1$ sufficiently close to one such that $cV_\tau(\Omega_V^\star)\le1.$ By \eqref{eq:tradeoff-risk-homogeneity}, we have $	V_\tau(c\Omega_V^\star) = cV_\tau(\Omega_V^\star) \le1,$ 	so $c\Omega_V^\star$ is feasible for the VaR-based problem. On the other hand, we can obtain 
\begin{align}
	-\frac12\log\det(c\Omega_V^\star)
	&=
	-\frac12\log\det(\Omega_V^\star)
	-
	\frac d2\log c
	\nonumber\\
	&<
	-\frac12\log\det(\Omega_V^\star),
\end{align}
which contradicts the optimality of $\Omega_V^\star$. Therefore, this leads to 	\begin{equation}
	V_\tau(\Omega_V^\star)=1.
	\label{eq:tradeoff-var-active}
\end{equation}
The same scaling argument applied to the CVaR-based problem gives	
\begin{equation}
	C_\tau(\Omega_C^\star)=1.
	\label{eq:tradeoff-cvar-active}
\end{equation}
By the definition of $\Gamma_{\tau,C}$ and \eqref{eq:tradeoff-cvar-active}, we have 
\begin{equation}
	\Gamma_{\tau,C}
	=
	\frac{
		C_\tau(\Omega_C^\star)
	}{
		V_\tau(\Omega_C^\star)
	}
	=
	\frac{1}{
		V_\tau(\Omega_C^\star)
	},
	\label{eq:tradeoff-cvar-var}
\end{equation}
Using positive homogeneity,
\begin{align}
	V_\tau
	\left(
	\Gamma_{\tau,C}\Omega_C^\star
	\right)=
	\Gamma_{\tau,C}
	V_\tau(\Omega_C^\star)=
	1.
	\label{eq:tradeoff-c-to-v-feasible}
\end{align}
Therefore, $\Gamma_{\tau,C}\Omega_C^\star$ is feasible for the VaR-based problem \eqref{eq:P-V-tau}. By the optimality of $\Omega_V^\star$, we have
\begin{align}
	-\frac12\log\det(\Omega_V^\star)
	&\le
	-\frac12
	\log\det
	\left(
	\Gamma_{\tau,C}\Omega_C^\star
	\right)\notag\\
	\implies \det(\Omega_V^\star)
	&\ge
	\det
	\left(
	\Gamma_{\tau,C}\Omega_C^\star
	\right).
	\label{eq:tradeoff-det-lower}
\end{align}

We now express the two VaR-normalized ellipsoids in unit-threshold form. By \eqref{eq:tradeoff-var-active}, $\mathcal E_{\tau,V}=\left\{r:r^\top\Omega_V^\star r\le1\right\}$. Similarly, by \eqref{eq:tradeoff-cvar-var}, we have $ \mathcal E_{\tau,C}  = \left\{ r:  r^\top  \left( \Gamma_{\tau,C}\Omega_C^\star \right) r\le1\right\}. $ 	For any $A\succ0$, the ellipsoid $\mathcal E(A) =\left\{r\in\mathbb R^d:r^\top A r\le1 \right\}$ has volume $\operatorname{Vol}(\mathcal E(A)) =\lambda(\mathbb B_d) \det(A)^{-1/2},$ where $\lambda(\mathbb B_d)$ denotes the volume of the $d$-dimensional Euclidean unit ball. Consequently,
\begin{equation}
	\frac{
		\operatorname{Vol}(\mathcal E_{\tau,C})
	}{
		\operatorname{Vol}(\mathcal E_{\tau,V})
	}
	=
	\left(
	\frac{
		\det(\Omega_V^\star)
	}{
		\det(
		\Gamma_{\tau,C}\Omega_C^\star
		)
	}
	\right)^{1/2}.
	\label{eq:tradeoff-volume-ratio}
\end{equation}	
Combining \eqref{eq:tradeoff-volume-ratio} with
\eqref{eq:tradeoff-det-lower} gives
\begin{equation}
	\frac{
		\operatorname{Vol}(\mathcal E_{\tau,C})
	}{
		\operatorname{Vol}(\mathcal E_{\tau,V})
	}
	\ge1.
	\label{eq:tradeoff-volume-lower}
\end{equation}

We next establish the upper bound. From \eqref{eq:tradeoff-var-active} and the definition of $\Gamma_{\tau,V}$, we have $C_\tau(\Omega_V^\star)=\Gamma_{\tau,V} V_\tau(\Omega_V^\star)= \Gamma_{\tau,V}.$ 	Therefore, by positive homogeneity, $C_\tau\left(
\frac{\Omega_V^\star}{\Gamma_{\tau,V}}\right) =\frac{C_\tau(\Omega_V^\star)}{
	\Gamma_{\tau,V}}=1.$ Hence, $\Omega_V^\star/\Gamma_{\tau,V}$ is feasible for the CVaR-based problem \eqref{eq:P-C-tau}. By the optimality of $\Omega_C^\star$, we can obtain
\begin{equation}
	-\frac12\log\det(\Omega_C^\star)
	\le
	-\frac12
	\log\det
	\left(
	\frac{\Omega_V^\star}
	{\Gamma_{\tau,V}}
	\right),
\end{equation} 
which equivalently leads to 
\begin{align}
	\det(\Omega_C^\star)\ge
	\det
	\left(
	\frac{\Omega_V^\star}
	{\Gamma_{\tau,V}}
	\right)=
	\Gamma_{\tau,V}^{-d}
	\det(\Omega_V^\star).
	\label{eq:tradeoff-det-upper-pre}
\end{align}
Multiplying both sides by $\Gamma_{\tau,C}^{d}$ yields
\begin{equation}
	\det
	\left(
	\Gamma_{\tau,C}\Omega_C^\star
	\right)
	\ge
	\left(
	\frac{
		\Gamma_{\tau,C}
	}{
		\Gamma_{\tau,V}
	}
	\right)^d
	\det(\Omega_V^\star).
	\label{eq:tradeoff-det-upper}
\end{equation}
Substituting \eqref{eq:tradeoff-det-upper} into
\eqref{eq:tradeoff-volume-ratio}, we obtain
\begin{align}
	\frac{
		\operatorname{Vol}(\mathcal E_{\tau,C})
	}{
		\operatorname{Vol}(\mathcal E_{\tau,V})
	}
	&\le
	\left[
	\frac{
		\det(\Omega_V^\star)
	}{
		\left(
		\Gamma_{\tau,C}/\Gamma_{\tau,V}
		\right)^d
		\det(\Omega_V^\star)
	}
	\right]^{1/2}
	\nonumber\\
	&=
	\left(
	\frac{
		\Gamma_{\tau,V}
	}{
		\Gamma_{\tau,C}
	}
	\right)^{d/2}.
	\label{eq:tradeoff-volume-upper}
\end{align}	
Combining \eqref{eq:tradeoff-volume-lower} and
\eqref{eq:tradeoff-volume-upper} gives
\begin{equation}
	1 \le \frac{ \operatorname{Vol}(\mathcal E_{\tau,C})}{ \operatorname{Vol}(\mathcal E_{\tau,V}) } \le\left(\frac{\Gamma_{\tau,V} }{ \Gamma_{\tau,C} }\right)^{d/2}.
	\label{eq:tradeoff-volume-final}
\end{equation}
This proves \eqref{eq:geometry-tail-bound}.

Moreover, since the volume ratio in \eqref{eq:tradeoff-volume-final} is at least one, the upper bound immediately implies $1\le
\left(\frac{\Gamma_{\tau,V}}{\Gamma_{\tau,C}}\right)^{d/2},$ and therefore $\Gamma_{\tau,C}\le \Gamma_{\tau,V}$. 	Thus, the CVaR-optimal geometry has a tail-severity ratio no larger than that of the VaR-optimal geometry.

Finally, suppose that $	\operatorname{Vol}(\mathcal E_{\tau,C}) =
\rho\,\operatorname{Vol}(\mathcal E_{\tau,V}), \rho\ge1.$ 	Then \eqref{eq:tradeoff-volume-upper} gives $\rho\le\left(\frac{\Gamma_{\tau,V}}{\Gamma_{\tau,C}} \right)^{d/2}.$ 	Raising both sides to the power $2/d$ gives
$\rho^{2/d}\le\frac{\Gamma_{\tau,V}}{\Gamma_{\tau,C}}.$ Rearranging yields $\Gamma_{\tau,C}\le\rho^{-2/d} \Gamma_{\tau,V}$. 	This proves \eqref{eq:gamma-volume-relation}.

\subsection{Details of Real-data Experiments}

For the Protein and House datasets, we construct response variables following \cite{feldman2023calibrated}. For all datasets, we first reserve half of the data to train the base predictor. The remaining data are further partitioned into training, calibration and test subsets with a ratio of 0.4, 0.3 and 0.3 for each trial. The training subset serves to estimate \(\Omega\),  while the calibration set is utilized for conformal prediction.

\subsubsection{Dataset details} 
The details of the used dataset are summarized as Table. \ref{tab:real dataset}
\begin{table}[h]
	\centering \footnotesize
	\caption{Experimental results on real data.}
	\label{tab:real dataset}
	\setlength{\tabcolsep}{4pt}
	\renewcommand{\arraystretch}{1.1}
	\begin{tabular}{cccc}
		\toprule
		Dataset  & \# Samples & Dimension of covariate & Dimension of responses  \\
		\midrule
		Gas  & 36733   & 9  & 2\\
		Protein  & 45730  &  8 & 2\\
		House  & 21613  & 18  & 2\\
		\bottomrule
	\end{tabular}
\end{table}

\subsubsection{Architecture and training of base predictor}

For all dataset, we simply train a 2-layer MLP and the size of the two layer are both 32. The activation function is  \textsc{ReLU}, solver is \textsc{adam}, and learning rate is 0.001.

\end{document}